%% file: longcat/copc.tex
\documentclass{article}

\usepackage{CJKutf8}
\usepackage{longcat/longcat_style}
\usepackage{adjustbox}
\usepackage[utf8]{inputenc} %
\usepackage[T1]{fontenc}    %
\usepackage{newunicodechar}
\usepackage{hyperref}       %
\usepackage{xcolor}
\usepackage[normalem]{ulem} %
\hypersetup{
    colorlinks=true,      %
    linkcolor=blue,      %
    urlcolor=blue,       %
    citecolor=blue,      %
    linkbordercolor=blue, %
    urlbordercolor=blue,
    citebordercolor=blue,
    pdfborderstyle={/S/U/W 1}, %
}
\usepackage{float}
\usepackage{url}            %
\usepackage{booktabs}       %
\usepackage{amsfonts}       %
\usepackage{nicefrac}       %
\usepackage{microtype}      %
\usepackage{lipsum}		%
\usepackage{graphicx}
\usepackage{xcolor}
\usepackage{colortbl}
\usepackage[most]{tcolorbox}
\usepackage{listings}
\usepackage{natbib}
\usepackage{doi}
\usepackage{amsmath}
\usepackage{amssymb} %
\usepackage{xspace}
\usepackage{enumitem}
\usepackage{multirow}
\usepackage{subcaption} 
\usepackage{makecell}
\usepackage{hyperref, cleveref}
\usepackage{pifont}
\usepackage[inkscapelatex=false]{svg}
\usepackage{caption}
\DeclareCaptionLabelSeparator{pipe}{ | }
\usepackage{tcolorbox}
\newtcolorbox{coloredquote}[1][]{
    colback=green!5!white,  
    colframe=green!70!black, 
    boxrule=2pt,
    arc=7pt,
    left=6pt,
    right=6pt,
    top=4pt,
    bottom=4pt,
    title=#1
}

\setlist[itemize]{leftmargin=*}
\setlist[enumerate]{leftmargin=*}
\setlist[description]{leftmargin=*}
\input{math_commands}

\definecolor{copcshade}{HTML}{EEECFF}
\definecolor{copcgain}{HTML}{009B63}

\definecolor{promptblue}{RGB}{66,127,175}
\definecolor{promptbg}{RGB}{235,245,251}

\newtcolorbox{promptbox}[1]{
    enhanced,
    breakable,
    colback=promptbg,
    colframe=promptblue,
    colbacktitle=promptblue,
    coltitle=white,
    title={#1},
    fonttitle=\bfseries\itshape\large,
    boxrule=1pt,
    arc=2.5mm,
    outer arc=2.5mm,
    titlerule=0pt,
    left=3mm,
    right=3mm,
    top=2.5mm,
    bottom=2.5mm,
    lefttitle=3mm,
    righttitle=3mm,
    toptitle=1.5mm,
    bottomtitle=1.5mm,
    before skip=6pt,
    after skip=6pt,
}

\lstdefinestyle{promptcode}{
    basicstyle=\ttfamily\small,
    columns=fullflexible,
    breaklines=true,
    keepspaces=true,
    showstringspaces=false,
    frame=none,
    backgroundcolor=\color{promptbg},
    xleftmargin=0pt,
    xrightmargin=0pt,
    aboveskip=2pt,
    belowskip=2pt
}

\title{COPC: Coupled Off-Policy Correction for Asynchronous LLM Reinforcement Learning}

\author{%
    \textbf{%
    Zicheng Hu\textsuperscript{1,2,3}\thanks{Work was done during the internship at Meituan.}\hspace{1.0em}%
    Zhijian Zhou\textsuperscript{1}\hspace{1.0em}%
    Xuan Zhang\textsuperscript{1,2}\hspace{1.0em}%
    Yuchen Liu\textsuperscript{1}\hspace{1.0em}%
    Cheng Chen\textsuperscript{3}\hspace{1.0em}%
    Yuan Li\textsuperscript{1,2}%
    } \\[2pt]
    \textbf{%
    Qi Gu\textsuperscript{2}\hspace{1.0em}%
    Yan Feng\textsuperscript{2}\hspace{1.0em}%
    Hongyan Hao\textsuperscript{2}\footnotemark[2]\hspace{1.0em}%
    Chao Qu\textsuperscript{1}\thanks{Hongyan Hao and Chao Qu are the corresponding authors.}%
    } \\[3pt]
    \textsuperscript{1}Fudan University \quad
    \textsuperscript{2}Meituan \quad
    \textsuperscript{3}East China Normal University \\
    \texttt{baodenayane@gmail.com \quad haohongyan02@meituan.com \quad quchao@fudan.edu.cn}
}

\renewcommand{\headeright}{\raisebox{-0.2\height}{\includegraphics[height=1.8em]{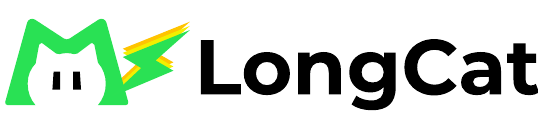}}}
\renewcommand{\shorttitle}{COPC: Coupled Off-Policy Correction}

\begin{document}
\maketitle
\setcounter{footnote}{0}

\begin{abstract}
    Asynchronous RL improves throughput in large language model post-training by decoupling rollout generation from optimization, at the cost of training on stale trajectories. Existing methods primarily address token-level policy mismatch through importance-ratio control in the actor objective, providing \emph{policy-side correction}. We show that this alone is insufficient because advantage estimates inherit mismatch from behavior-policy continuations, a phenomenon we term \emph{advantage staleness}. This motivates a complementary \emph{advantage-side correction}, raising the question: how do the two correction channels interact, and how should their strengths be coordinated?
    We derive exact bias and variance decompositions for a general two-channel actor update, revealing a nonseparable coupling between policy-weight and advantage-estimation errors: their interaction induces multiplicative bias terms, while squared policy weights amplify advantage uncertainty in gradient variance. This coupling motivates a testable hypothesis that the preferred settings of the two correction channels should be coordinated rather than chosen independently. Guided by this analysis, we introduce \textbf{C}oupled \textbf{O}ff-\textbf{P}olicy \textbf{C}orrection (\textbf{COPC}), an actor--critic method combining token-level ratio masking with a dedicated advantage-side correction based on two-sided clipped-ratio weighting of TD residuals for return and advantage estimation.
    Joint parameter sweeps across staleness levels provide empirical support for this hypothesis: the effect of varying one correction parameter depends on, and can even reverse with, the setting of the other. COPC achieves the highest reported performance on both tool-integrated mathematical reasoning and search, outperforming the strongest reported asynchronous baseline in each setting. COPC further exhibits strong practical robustness, with a broad high-performing parameter region and substantially improved training stability. In particular, in the search setting, it remains stable throughout training, while most evaluated asynchronous baselines exhibit late-stage performance collapse. These performance gains persist at 64-step policy staleness, while COPC incurs minimal step-time overhead over asynchronous PPO and retains a $1.7\times$ step-time speedup over synchronous PPO.
\end{abstract}

\input{sections/introduction}

\input{sections/releated_work}

\input{sections/preliminaries}

\input{sections/theory}

\input{sections/algorithm}

\input{sections/experiments}

\input{sections/conclusion}

\bibliographystyle{unsrtnat}
\bibliography{longcat/ref}

\clearpage

\appendix
\input{sections/appendix}

\end{document}

%% file: math_commands.tex
\usepackage{amsmath,amsfonts,bm}

\def\eqref#1{equation~\ref{#1}}

\def\1{\bm{1}}

\DeclareMathAlphabet{\mathsfit}{\encodingdefault}{\sfdefault}{m}{sl}
\SetMathAlphabet{\mathsfit}{bold}{\encodingdefault}{\sfdefault}{bx}{n}

\usepackage{amsmath}\allowdisplaybreaks
\usepackage{amsfonts,bm}
\usepackage{amssymb}
\usepackage{amsthm}
\usepackage{algorithm}
\usepackage{algorithmic}

\theoremstyle{plain}

\newtheorem{remark}{Remark}

\newtheorem{corollary}{Corollary}
\newtheorem{proposition}{Proposition}

\def\cD{{\mathcal{D}}}
\def\cE{{\mathcal{E}}}

\def\cL{{\mathcal{L}}}

\def\BE{{\mathbb{E}}}

\def\BI{{\mathbb{I}}}

%% file: sections/introduction.tex
\section{Introduction}
\label{sec:introduction}

Reinforcement learning (RL) has become a central paradigm for post-training large language models (LLMs) with preference feedback and verifiable rewards~\citep{ouyang2022training,shao2024deepseekmath,guo2025deepseek}. Scaling RL, however, is often limited by costly rollouts whose completion times vary with response length and tool interactions, leading to substantial idle time under synchronous training. Asynchronous RL alleviates this bottleneck by decoupling rollout generation from policy optimization, allowing both stages to proceed concurrently~\citep{wu2025llamarl,slime_github,fu2026areal}. This overlap improves hardware utilization but introduces policy staleness, as the learner updates on trajectories generated by earlier policy versions. Learning effectively under this behavior--learner mismatch is therefore central to realizing the benefits of asynchronous RL.

Policy lag can destabilize optimization as trajectories become increasingly off-policy relative to the current learner policy~\citep{xu2026gac,huang2026stable}. Recent asynchronous LLM-RL methods primarily mitigate this mismatch through token-level importance-ratio control in the actor objective, using clipping, masking, or related variance-control mechanisms~\citep{guan2026missing,zhao2026deconstructing}; we refer to these mechanisms as \emph{policy-side correction}. However, policy-side correction does not correct advantage estimation: the advantage still depends on continuations generated by the stale behavior policy. Hence, a token can receive an appropriate policy-side weight while its advantage remains tied to stale behavior-policy continuations. We term this \emph{advantage staleness}.

Advantage staleness motivates a complementary \emph{advantage-side correction} that acts directly on advantage estimation. This yields two distinct channels in the actor update: policy-side correction controls how off-policy samples are weighted, whereas advantage-side correction controls how their learning signals are estimated. This leads to a central question:
\par
{\centering
\emph{Do the two correction channels interact, and how should their strengths be coordinated?}
\par}
\noindent
We answer the first part by deriving exact bias and variance decompositions for a general two-channel actor update. The resulting expressions reveal a nonseparable interaction: policy-weight and advantage-estimation errors couple multiplicatively in the bias, while squared policy weights amplify conditional advantage uncertainty in the gradient variance. Thus, the two correction channels generally cannot be calibrated independently. Although the analysis does not prescribe an optimal coordination rule, it motivates a concrete, testable hypothesis: stronger policy-side mismatch control should favor weaker advantage-side correction, and vice versa.

Guided by this analysis, we introduce \textbf{C}oupled \textbf{O}ff-\textbf{P}olicy \textbf{C}orrection (\textbf{COPC}), an actor--critic method that instantiates the two channels with token-level ratio masking on the policy side and two-sided clipped-ratio weighting of TD residuals on the advantage side. The latter produces corrected returns for critic training and actor advantage estimation, while exposing a separate strength parameter that can be jointly calibrated with the policy-side correction in practice.

We test the coordination hypothesis through systematic two-dimensional parameter sweeps across multiple staleness levels. Figure~\ref{fig:coupled_sweeps} shows that preferred configurations follow the hypothesized coordination pattern, while the effect of varying one correction parameter can improve performance under one setting of the other but degrade it under another. This pattern persists across staleness levels, providing empirical support for both the hypothesized coordination and the cross-channel coupling predicted by our analysis. At the same time, strong performance spans a broad region of the parameter space, indicating robustness to precise hyperparameter choices. Given this robustness, we use the same pair of correction parameters across all tasks in our main experiments.

COPC achieves the highest reported performance on tool-integrated mathematical reasoning and search, measured by avg@32 and greedy-decoding accuracy, respectively, surpassing the strongest reported asynchronous baseline by 2.03 and 2.77 points. It substantially improves training stability: in the search setting, COPC remains stable throughout training, while most evaluated asynchronous baselines exhibit late-stage performance collapse. These gains persist at 64-step policy staleness, with minimal overhead over asynchronous PPO and a $1.7\times$ speedup over synchronous PPO.

Our contributions are fourfold:
\begin{itemize}
    \item \textbf{Advantage staleness and two-channel analysis.}
    We identify \emph{advantage staleness} as a distinct off-policy mismatch that persists after policy-side correction. Exact bias and variance decompositions for a general two-channel actor update reveal a nonseparable coupling between policy- and advantage-side errors, motivating a testable coordination hypothesis.

    \item \textbf{Coupled off-policy correction.}
    We introduce COPC, an actor--critic method combining token-level ratio masking with a new advantage-side correction based on two-sided clipped-ratio weighting of TD residuals. Corrected returns are used for critic training and actor advantage estimation, with separate channel parameters enabling joint calibration.

    \item \textbf{Empirical validation of coordinated correction.}
    Joint parameter sweeps across staleness levels support the coordination hypothesis and reveal pronounced cross-channel interactions. They also identify a broad high-performing parameter region, supporting a single shared setting across tasks without task-specific retuning in our experiments.

    \item \textbf{Performance, stability, and efficiency.}
    With the shared setting, COPC achieves the highest reported performance on both tool-integrated mathematical reasoning and search, while substantially improving training stability; in the search setting, it remains stable throughout training while most evaluated asynchronous baselines exhibit late-stage performance collapse. These gains persist at 64-step policy staleness, with minimal overhead over asynchronous PPO and a $1.7\times$ speedup over synchronous PPO across both evaluation settings.
\end{itemize}

%% file: sections/releated_work.tex
\section{Related Work}
\label{sec:related_work}

\paragraph{Asynchronous RL for LLM post-training.}
Asynchronous actor--learner architectures improve throughput by decoupling data collection from optimization, at the cost of behavior--learner policy mismatch~\citep{espeholt2018impala}. Recent systems extend this paradigm to LLM post-training through distributed rollout generation and asynchronous policy updates, including asynchronous RLHF, AReaL, and LlamaRL~\citep{noukhovitch2025asynchronous,fu2026areal,wu2025llamarl}. Our work focuses on how to reliably correct learner updates under the resulting policy staleness.

\paragraph{Policy-side stabilization in off-policy LLM optimization.}
Recent LLM-RL methods mitigate behavior--learner mismatch through importance-ratio control and related stabilization mechanisms. AReaL separates behavior-policy correction from proximal policy constraints, while SAO applies two-sided token-level ratio masking to suppress updates under large policy mismatch~\citep{fu2026areal,hou2026single}. Other approaches target variance under policy staleness: VCPO combines effective-sample-size-based learning-rate adaptation with a minimum-variance baseline, whereas M2PO constrains the second moment of importance weights to limit high-variance updates~\citep{huang2026stable,zheng2025prosperity}. These methods primarily stabilize off-policy actor updates, whereas COPC additionally corrects continuation-dependent returns and advantages.

\paragraph{Off-policy return and advantage correction.}
Classical off-policy RL methods correct multi-step return estimation through importance weighting or truncation. Retrace constructs off-policy return targets using truncated importance ratios, while V-trace uses corrected returns for both critic learning and actor advantage estimation~\citep{munos2016safe,espeholt2018impala}. Reactor studies how truncation and value-estimation error affect off-policy actor--critic updates, and C-trace characterizes trade-offs among update variance, fixed-point bias, and contraction rate~\citep{gruslys2017reactor,rowland2020adaptive}. COPC differs in both correction design and scope: it introduces an advantage-side correction based on two-sided clipped-ratio weighting of TD residuals, with a separate parameter from policy-side masking, and studies the interaction between the two correction channels.

%% file: sections/preliminaries.tex
\section{Preliminaries and Motivations}
\label{sec:preliminaries}

\subsection{Asynchronous RL Training}

We model an LLM as an autoregressive policy $\pi_\theta$ generating a response $y=(y_1,\ldots,y_{|y|})$ to a prompt $x\sim\mathcal D$, where $s_t$ denotes the context preceding $y_t$. In asynchronous RL, the current \emph{learner policy} $\pi_\theta$ is updated using responses generated by a stale \emph{behavior policy} $\mu$. PPO-style optimization then combines importance ratios with advantage estimates in a clipped surrogate objective:
\begin{equation}
\label{eq:ppo_surrogate}
    \mathcal J(\theta)=\mathbb E_{x\sim\mathcal D,\,y\sim\mu(\cdot\mid x)}\!\bigg[\frac{1}{|y|}\sum_{t=1}^{|y|}\min\!\left(\rho_t(\theta)\hat A_t,\operatorname{clip}\!\left(\rho_t(\theta),1-\epsilon,1+\epsilon\right)\hat A_t\right)\bigg].
\end{equation}
Here, the token-level ratio
$\rho_t(\theta)=\frac{\pi_\theta(y_t\mid s_t)}{\mu(y_t\mid s_t)}$
measures learner--behavior policy mismatch at token $t$, $\hat A_t$ denotes the advantage estimate used by the actor, and $\epsilon$ controls the clipping range.

\paragraph{Proximal Policy Optimization (PPO).}
PPO uses a learned critic $V_\phi(s_t)$ to estimate state values. Letting $V_t=V_\phi(s_t)$, generalized advantage estimation (GAE)~\citep{schulman2015high} constructs advantages from temporal-difference (TD) residuals:
\begin{equation}
\label{eq:gae}
    \hat A_t^{\mathrm{GAE}}
    =
    \sum_{j=t}^{|y|}
    (\gamma\lambda)^{j-t}\delta_j,
    \qquad
    \delta_j
    =
    r_j+\gamma V_{j+1}-V_j.
\end{equation}
Here, $r_j$ denotes the reward at position $j$, $\gamma$ is the discount factor, and $\lambda$ controls trace decay and the associated bias--variance trade-off. We set the terminal value to $V_{|y|+1}=0$.

\subsection{Why Advantage-Side Correction Is Needed}
To isolate the role of advantage-side correction, we reweight TD residuals along behavior-policy trajectories using two-sided clipped importance ratios. Let
$\bar\rho_j=\operatorname{clip}\bigl(\rho_j(\theta),1-\epsilon_\ell,1+\epsilon_h\bigr)$
denote the clipped ratio. We define the corrected return and induced advantage as
\begin{equation}
\label{eq:advantage_side_correction}
    \hat v_t=V_t+\sum_{j=t}^{|y|}(\gamma\lambda)^{j-t}\bar\rho_j\delta_j,\qquad
    \hat A_t=\delta_t+\sum_{j=t+1}^{|y|}(\gamma\lambda)^{j-t}\bar\rho_j\delta_j.
\end{equation}
Setting $\bar\rho_j=1$ recovers standard GAE in~\eqref{eq:gae}, while the corresponding expected return operator has fixed point $V^\mu$. Setting $\bar\rho_j=\rho_j(\theta)$ instead applies the exact, unclipped importance ratio and shifts the fixed point to $V^{\pi_\theta}$, correcting the mismatch induced by stale behavior-policy continuations but potentially increasing estimator variance. Two-sided clipping bounds extreme ratios and controls this variance, at the cost of shifting the fixed point away from $V^{\pi_\theta}$. The following proposition characterizes the resulting effective policy; its proof is provided in Appendix~\ref{app:return_fixed_point}.
\begin{proposition}[Unique fixed point]
\label{prop:return_fixed_point}
Under the finite-horizon and support assumptions, the expected corrected-return operator associated with~\eqref{eq:advantage_side_correction} has the unique fixed point $V^{\widetilde\pi}$, where
\begin{equation}
\label{eq:effective_correction_policy}
    \widetilde\pi(a\mid s)=\frac{\mu(a\mid s)\bar\rho(s,a)}{\sum_b\mu(b\mid s)\bar\rho(s,b)}.
\end{equation}
\end{proposition}
Proposition~\ref{prop:return_fixed_point} makes the source of advantage staleness explicit. Without advantage-side correction, $\bar\rho(s,a)=1$ and hence $\widetilde\pi=\mu$, so the return target and induced advantage remain tied to stale behavior-policy continuations rather than the learner policy $\pi_\theta$. Thus, policy-side correction alone cannot eliminate advantage staleness. Consistent with this prediction, COPC w/o RC replaces the corrected returns and advantages with standard GAE while keeping other settings fixed, and underperforms COPC on both tool-integrated mathematical reasoning and search (Tables~\ref{tab:retool_main_results} and~\ref{tab:searchr1_main_results}).

\subsection{Two-Channel Correction Setup}
We formalize the two correction channels through an idealized reference actor update. Let $A_t^{\pi_\theta}$ denote the advantage under current-policy continuations and $\psi_t(\theta)=\nabla_\theta\log\pi_\theta(y_t\mid s_t)$ the token-level score function. Using the raw importance ratio and current-policy advantage, we define
\begin{equation}
    \label{eq:reference_gradient}
    g^\star(\theta)=\mathbb E_{x\sim\mathcal D,y\sim\mu(\cdot\mid x)}\!\bigg[\frac{1}{|y|}\sum_{t=1}^{|y|}\rho_t(\theta)\psi_t(\theta)A_t^{\pi_\theta}\bigg].
\end{equation}
This reference isolates the two quantities that must be approximated or controlled in practice: the policy weight $\rho_t(\theta)$ and the continuation-dependent advantage $A_t^{\pi_\theta}$. We parameterize their respective corrections by $\eta_P$ and $\eta_A$ to study their interaction in the actor update.

\paragraph{Policy-side correction.}
Policy-side correction replaces the raw importance ratio $\rho_t(\theta)$ with a controlled weight $w_t(\eta_P)$. For example, SAO~\citep{hou2026single} uses ratio masking:
\begin{equation*}
    w_t(\eta_P)=\rho_t(\theta)\,\BI\!\left\{1-\epsilon_\ell^P\le\rho_t(\theta)\le1+\epsilon_h^P\right\}.
\end{equation*}
Here, $\eta_P=(\epsilon_\ell^P,\epsilon_h^P)$ defines the masking interval: $w_t=\rho_t(\theta)$ inside and $w_t=0$ outside.

\paragraph{Advantage-side correction.}
Advantage-side correction replaces the ideal current-policy advantage $A_t^{\pi_\theta}$ with an estimate $\hat A_t(\eta_A)$ that accounts for mismatch from stale behavior-policy continuations, where $\eta_A$ parameterizes the correction. For example, in~\eqref{eq:advantage_side_correction}, $\eta_A=(\epsilon_\ell^A,\epsilon_h^A)$ specifies the clipping interval for importance ratios used to weight TD residuals.

Combining the two correction channels yields the expected length-normalized actor update
\begin{equation}
\label{eq:expected_two_sided_gradient}
    g(\eta_P,\eta_A)=\BE_{x\sim\cD,\,y\sim\mu(\cdot\mid x)}\!\bigg[\frac{1}{|y|}\sum_{t=1}^{|y|}w_t(\eta_P)\psi_t(\theta)\hat A_t(\eta_A)\bigg].
\end{equation}

%% file: sections/theory.tex
\section{Coupling Between Correction Channels}
\label{sec:theory}

The preceding fixed-point analysis characterizes the effective value target induced by a concrete advantage-side correction. We now study, more generally, how advantage-side correction interacts with policy-side ratio control in the actor update. We derive exact bias and variance decompositions relative to the reference update $g^\star(\theta)$ to characterize how errors from the two correction channels interact. The analysis treats $w_t(\eta_P)$ and $\hat A_t(\eta_A)$ abstractly under the conditions below, without assuming a particular correction mechanism. All quantities are evaluated at a fixed training iterate, with all required moments assumed finite. Proofs are provided in Appendix~\ref{proof:couple_bias_variance}.

\subsection{Bias Coupling}
For channel parameters $(\eta_P,\eta_A)$, define the policy-weight and advantage-estimation errors as
\begin{equation*}
    \Delta_t(\eta_P)=w_t(\eta_P)-\rho_t(\theta),\qquad
    e_t(\eta_A)=\hat A_t(\eta_A)-A_t^{\pi_\theta}.
\end{equation*}
\begin{proposition}[Bias coupling]
\label{prop:bias_decomposition}
The update bias relative to the reference satisfies
\begin{equation}
\label{eq:exact_bias_decomposition}
    g(\eta_P,\eta_A)-g^\star(\theta)
    =\BE\bigg[\frac{1}{|y|}\sum_{t=1}^{|y|}
    \psi_t(\theta)\Big(
    \underbrace{\Delta_t(\eta_P)A_t^{\pi_\theta}}_{\text{policy-side}}
    +\underbrace{\rho_t(\theta)e_t(\eta_A)}_{\text{advantage-side}}
    +\underbrace{\Delta_t(\eta_P)e_t(\eta_A)}_{\text{interaction}}
    \Big)\bigg].
\end{equation}
\end{proposition}
The first two terms capture the direct contributions of policy-weight and advantage-estimation errors. The interaction term $\Delta_t(\eta_P)e_t(\eta_A)$ reveals their multiplicative coupling: its expected contribution can reinforce or offset the direct terms. Consequently, the effect of adjusting either correction channel on the overall update bias generally depends on the residual error in the other.

\subsection{Variance Coupling}
For a random vector $X$, let $\mathbb V[X]=\BE[|X-\BE[X]|_2^2]$ denote its total variance. Define the token-level gradient contribution as
\begin{equation*}
    Z_t=\frac{w_t(\eta_P)\psi_t(\theta)\hat A_t(\eta_A)}{|y|},
\end{equation*}
with $Z_t=0$ beyond the response length. Let $C_t$ denote conditioning information under which $|y|$, $\psi_t(\theta)$, and $w_t(\eta_P)$ are measurable for every policy-side setting $\eta_P$. Using the same conditioning information across correction settings, we define the conditional moments
\begin{equation*}
    m_t(\eta_A)=\BE[\hat A_t(\eta_A)\mid C_t],\qquad
    \sigma_t^2(\eta_A)=\operatorname{Var}[\hat A_t(\eta_A)\mid C_t].
\end{equation*}
\begin{proposition}[Variance coupling]
\label{prop:variance_decomposition}
The variance of each token contribution decomposes as
\begin{equation}
\label{eq:variance_decomposition}
    \mathbb V[Z_t]
    =\mathbb V\!\left[\frac{w_t(\eta_P)\psi_t(\theta)m_t(\eta_A)}{|y|}\right]
    +\BE\!\left[\frac{w_t^2(\eta_P)\|\psi_t(\theta)\|_2^2\sigma_t^2(\eta_A)}{|y|^2}\right].
\end{equation}
\end{proposition}
The first term captures conditional-mean variation. The second exposes coupling through $w_t^2(\eta_P)\sigma_t^2(\eta_A)$: squared policy weights rescale advantage uncertainty in gradient variance. The full-response decomposition, including cross-token covariances, is given in Appendix~\ref{app:trajectory_variance}.

\subsection{Joint Gradient-Estimation Error}
Consider $N$ prompt groups, each containing $K$ dependent responses. Let $Z_{i,k}$ denote the gradient contribution of response $k$ to prompt $i$, and define group average and finite-sample update as
\begin{equation*}
    \bar Z_{i,K}=\frac{1}{K}\sum_{k=1}^{K}Z_{i,k},\qquad
    \hat g_{N,K}=\frac{1}{N}\sum_{i=1}^{N}\bar Z_{i,K}.
\end{equation*}
We assume the group averages are independent and identically distributed while allowing dependence among responses within each group. Let $\bar Z_K$ denote a generic group average, with $\BE[\bar Z_K]=g(\eta_P,\eta_A)$ under the sampling and advantage-estimation scheme defining $g$.

\begin{corollary}[Joint gradient error]
\label{cor:joint_gradient_error}
For fixed $(\eta_P,\eta_A)$, the update has mean-squared error as follows
\begin{equation}
\label{eq:joint_gradient_mse}
    \mathcal E(\eta_P,\eta_A)
    :=\BE\!\left[\|\hat g_{N,K}-g^\star(\theta)\|_2^2\right]
    =\|g(\eta_P,\eta_A)-g^\star(\theta)\|_2^2
    +\frac{1}{N}\mathbb V[\bar Z_K].
\end{equation}
\end{corollary}

\begin{remark}[Implications for joint calibration]
\label{rem:joint_calibration}
The preceding decompositions expose cross-channel interactions in both update bias and variance through $\Delta_t e_t$ and $w_t^2\sigma_t^2$. Consequently, the gradient-estimation error in~\eqref{eq:joint_gradient_mse} need not be separable in $\eta_P$ and $\eta_A$ at a fixed training iterate, motivating joint rather than independent calibration of the two correction channels.
\end{remark}

\paragraph{Coordination hypothesis.}
This nonseparable coupling motivates a concrete directional hypothesis: stronger correction in one channel should favor weaker correction in the other, and vice versa. We test this hypothesis through two-dimensional parameter sweeps across multiple staleness levels. Under the COPC parameterization in Section~\ref{sec:algorithm}, smaller values of $\eta_P$ and $\eta_A$ correspond to stronger correction in their respective channels. As shown in Figure~\ref{fig:coupled_sweeps}, the sweeps are consistent with the hypothesized coordination and reveal pronounced cross-channel interactions: changing one correction parameter can improve performance under one setting of the other yet degrade it under another.

\begin{figure}
    \centering
    \includegraphics[width=\linewidth]{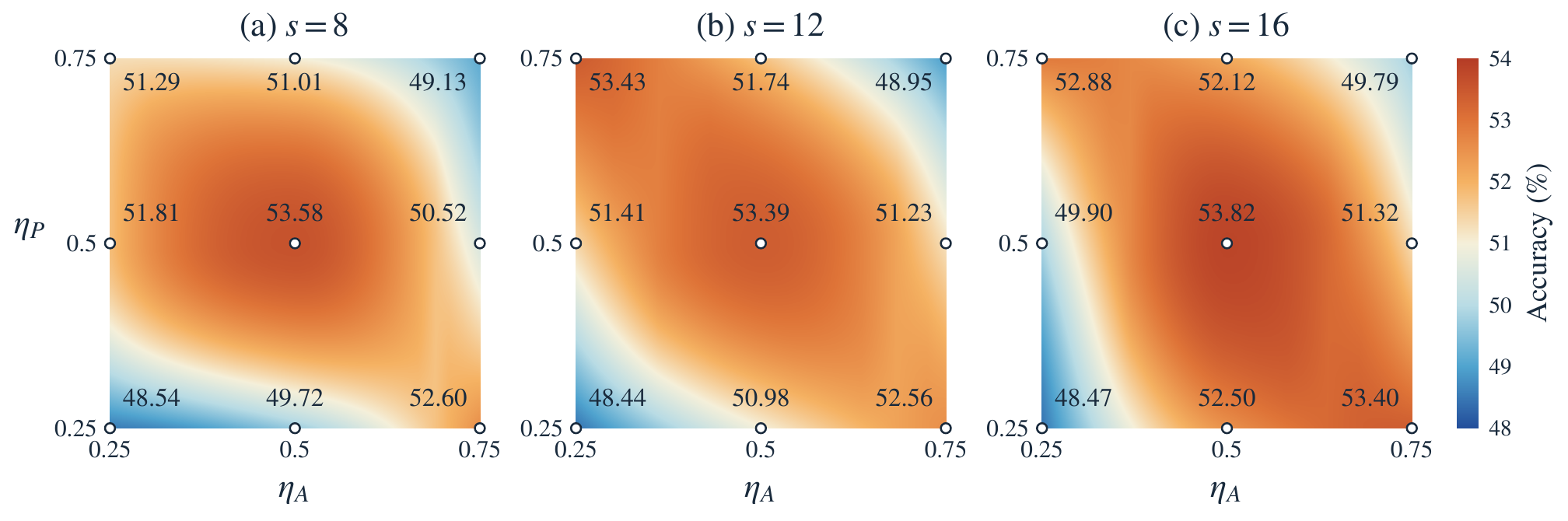}
    \caption{Mean avg@32 accuracy on AIME 2025 under joint policy- and advantage-side correction sweeps at staleness levels $s\in\{8,12,16\}$.}
    \label{fig:coupled_sweeps}
\end{figure}

%% file: sections/algorithm.tex
\section{COPC: Coupled Off-Policy Correction}
\label{sec:algorithm}

We introduce COPC, an actor--critic method that combines our proposed advantage-side correction in~\eqref{eq:advantage_side_correction} with the policy-side masking rule of SAO~\citep{hou2026single}. The corrected returns serve as critic targets and construct actor advantages, while masking controls how these learning signals contribute to the policy update. Motivated by the coupling analysis in Section~\ref{sec:theory}, COPC treats the two correction strengths as jointly selected hyperparameters.

\paragraph{Joint parameterization.}
We parameterize the policy- and advantage-side correction intervals using fixed base margins $\epsilon_\ell,\epsilon_h>0$ and separate interval controls $\eta_P>0$ and $\eta_A\ge0$:
\begin{equation*}
    \ell_X=1-\eta_X\epsilon_\ell,\qquad
    u_X=1+\eta_X\epsilon_h,\qquad
    X\in\{P,A\}.
\end{equation*}
Here, correction strength refers to the restrictiveness of the ratio constraint: smaller $\eta_X$ narrows the interval and represents stronger correction under this convention. The bounds $(\ell_P,u_P)$ control policy-side masking, while $(\ell_A,u_A)$ control advantage-side ratio clipping. Thus, smaller $\eta_P$ yields stricter masking, whereas smaller $\eta_A$ clips TD-residual importance weights more tightly around $1$. We select $(\eta_P,\eta_A)$ jointly to account for their coupled effect on the actor update.

\paragraph{Actor--critic update.}
For a rollout batch $\mathcal B$, before advantage normalization, the actor objective is
\begin{equation}
\label{eq:copc_actor_loss}
    \cL_{\mathrm{actor}}(\theta)=-\frac{1}{|\mathcal B|}\sum_{(x,y)\in\mathcal B}\frac{1}{|y|}\sum_{t=1}^{|y|}\rho_t(\theta)\,\BI\!\left\{\ell_P<\rho_t(\theta)<u_P\right\}\operatorname{sg}\!\left[\hat A_t(\eta_A)\right].
\end{equation}
Here, $\operatorname{sg}[\cdot]$ denotes stop-gradient. The indicator is treated as a fixed, non-differentiable mask: masked tokens contribute zero actor gradient, whereas retained tokens preserve the gradient path through the importance-ratio term $\rho_t(\theta)$, without differentiating through the masking decision.

The critic learns from the corrected returns via clipped value regression. Let $V_\phi(s_t)$ denote the current value prediction and $V_t$ the frozen pre-update prediction. With threshold $\epsilon_V>0$, define
\begin{equation*}
    V_t^{\mathrm{clip}}(\phi)=V_t+\operatorname{clip}\!\left(V_\phi(s_t)-V_t,-\epsilon_V,\epsilon_V\right).
\end{equation*}
The critic is optimized with the clipped regression objective
\begin{equation}
\label{eq:copc_critic_loss}
    \cL_{\mathrm{critic}}(\phi)=\frac{1}{|\mathcal B|}\sum_{(x,y)\in\mathcal B}\frac{1}{|y|}\sum_{t=1}^{|y|}\max\!\left\{\left(V_\phi(s_t)-\operatorname{sg}[\hat v_t]\right)^2,\left(V_t^{\mathrm{clip}}(\phi)-\operatorname{sg}[\hat v_t]\right)^2\right\}.
\end{equation}
Return targets and actor advantages are computed from frozen pre-update quantities and held fixed during optimization. Policy-side masking acts only on the actor gradient: masked tokens contribute to return estimation and critic training, and their TD residuals can influence earlier advantages.

%% file: sections/experiments.tex
\section{experiments}
\label{sec:experiments}

\begin{table*}[t]
\centering
\caption{Performance on tool-integrated mathematical reasoning measured by avg@32. Best in \textbf{bold} and second-best \underline{underlined} within each backbone group. $\Delta$ rows report gains of COPC over PPO.}
\vspace{3pt}
\label{tab:retool_main_results}
\setlength{\tabcolsep}{3pt}
    \begin{tabular*}{\textwidth}{@{\extracolsep{\fill}}lcccccc@{}}
    \toprule
    \textbf{Method} & \textbf{AIME24} & \textbf{AIME25} & \textbf{AIME26} & \textbf{BRUMO25} & \textbf{BeyondAIME} & \textbf{Avg.} \\
    \midrule
    \multicolumn{7}{@{}l}{\textit{Qwen3-4B-Instruct}} \\
    \addlinespace
    PPO & 54.65 & 43.75 & 47.40 & 44.83 & 27.53 & 43.63 \\
    V-trace & 55.07 & 46.28 & 50.17 & 46.42 & 28.80 & 45.35 \\
    DAPO & 56.42 & 48.99 & 51.11 & 46.08 & 29.17 & 46.35 \\
    SAO & 58.06 & 50.31 & 53.13 & \underline{47.29} & 30.48 & 47.85 \\
    AReaL & \underline{58.40} & \underline{51.25} & \underline{53.47} & 47.01 & \underline{31.19} & \underline{48.26} \\
    \midrule
    COPC w/o RC & 58.33 & 50.83 & 53.33 & 47.22 & 30.82 & 48.11 \\
    \textbf{COPC (ours)} & \textbf{60.73} & \textbf{53.58} & \textbf{55.63} & \textbf{49.27} & \textbf{32.22} & \textbf{50.29} \\
    $\Delta$ & +6.08 & +9.83 & +8.23 & +4.44 & +4.69 & +6.65 \\
    \bottomrule
    \end{tabular*}
\end{table*}

\begin{table*}[t]
\centering
\caption{Performance on search measured by greedy-decoding accuracy. Best in \textbf{bold} and second-best \underline{underlined} within each backbone group. $\Delta$ reports gains of COPC over PPO.}
\vspace{3pt}
\label{tab:searchr1_main_results}
\setlength{\tabcolsep}{2pt}
    \begin{tabular*}{\textwidth}{@{\extracolsep{\fill}}lcccccccc@{}}
    \toprule
    \textbf{Method} & \textbf{NQ} & \textbf{TriviaQA} & \textbf{PopQA} & \textbf{HotpotQA} & \textbf{2Wiki} & \textbf{MuSiQue} & \textbf{Bamboogle} & \textbf{Avg.} \\
    \midrule
    \multicolumn{9}{@{}l}{\textit{Llama-3.2-3B-Instruct}} \\
    \addlinespace
    PPO & 41.00 & 52.20 & 37.20 & 37.60 & 27.00 & 26.00 & 40.80 & 37.40 \\
    V-trace & 40.40 & 55.40 & 36.80 & 38.20 & 25.20 & 26.20 & 42.40 & 37.80 \\
    DAPO & \underline{44.00}& 54.60 & 36.80 & 39.40 & 39.80 & 28.80 & 45.60 & 41.29 \\
    SAO & 39.80 & 48.20 & \underline{39.20} & 37.40 & 40.20 & 27.20 & \underline{46.40} & 39.77 \\
    AReaL & 40.80 & 52.80 & \underline{39.20} & \underline{43.20} & \textbf{41.00} & \underline{30.80} & 44.80 & \underline{41.80} \\
    \midrule
    COPC w/o RC & 43.00 & \underline{56.20} & \textbf{39.80} & 42.00 & 29.40 & 28.80 & \underline{46.40} & 40.80 \\
    \textbf{COPC (ours)} & \textbf{45.80} & \textbf{57.20} & 38.80 & \textbf{47.60} & \underline{40.40} & \textbf{32.60} & \textbf{49.60} & \textbf{44.57} \\
    $\Delta$ & +4.80 & +5.00 & +1.60 & +10.00 & +13.40 & +6.60 & +8.80 & +7.17 \\
    \bottomrule
    \end{tabular*}
\end{table*}

\begin{figure}[t]
    \centering
    \includegraphics[width=0.48\linewidth]{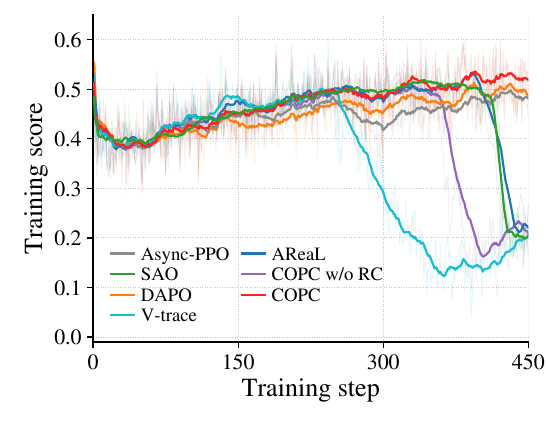}
    \hfill
    \includegraphics[width=0.48\linewidth]{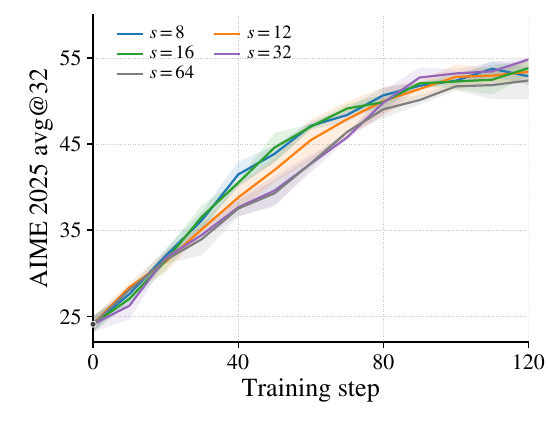}
    \caption{Training stability and robustness. Left: training dynamics of asynchronous methods on search. Right: COPC training dynamics on AIME 2025 under different policy-staleness levels.}
    \label{fig:stability_robustness}
\end{figure}

We evaluate COPC along four dimensions. Joint parameter sweeps test the coordination hypothesis across staleness levels, while experiments on tool-integrated mathematical reasoning and search compare COPC with synchronous and asynchronous baselines and isolate the contribution of advantage-side correction. We further assess robustness to policy staleness and training efficiency through step time, throughput, and step-time speedup over synchronous PPO. Detailed training and inference configurations, including method-specific settings, are provided in Appendix~\ref{ape:detailed_training_configurations}.

\subsection{Experimental setup}
All experiments are implemented in Slime~\citep{slime_github}, which supports asynchronous RL by overlapping rollout generation with policy optimization. Within each task setting, all methods share the same rollout and evaluation implementations while retaining their method-specific objectives and return or advantage estimators.

We use a strengthened PPO training recipe incorporating selected practices from VAPO~\citep{yue2025vapo}, including Clip-Higher, critic pretraining, group sampling, advantage normalization, and clipped value regression. We use $\gamma=\lambda=1$ throughout all experiments. These choices provide a training backbone while keeping the comparison focused on policy- and advantage-side correction.

\paragraph{Baselines.}
We compare COPC with representative asynchronous RL baselines and a controlled ablation. PPO~\citep{schulman2017proximal} uses the clipped surrogate objective with GAE advantages and follows our strengthened PPO recipe incorporating selected practices from VAPO~\citep{yue2025vapo}, serving as a strong reference without explicit off-policy correction. DAPO~\citep{yu2025dapo} provides a critic-free alternative with group-relative advantages and asymmetric ratio clipping. AReaL~\citep{fu2026areal} separates behavior-policy correction from proximal policy constraints, while SAO~\citep{hou2026single} applies token-level ratio masking to suppress updates under large policy mismatch. V-trace~\citep{espeholt2018impala} uses truncated importance ratios for corrected return estimation and actor advantage construction. We additionally report COPC w/o RC, which replaces COPC's corrected returns and advantages with standard GAE while keeping all other settings fixed, isolating the contribution of advantage-side correction beyond policy-side masking.

\paragraph{Tool-integrated mathematical reasoning.}
We use Qwen3-4B-Instruct~\citep{yang2025qwen3} as the policy backbone. Following the released ReTool training recipe~\citep{feng2025retool}, we initialize the policy with supervised fine-tuning on the ReTool SFT dataset and perform RL training on DAPO-Math-17K~\citep{yu2025dapo}. We evaluate on AIME 2024/2025/2026~\citep{aime}, BeyondAIME~\citep{bytedance_seed_2025_beyondaime}, and BrUMO~\citep{balunovic2026matharena}.

\paragraph{Search.}
We use Llama-3.2-3B-Instruct~\citep{grattafiori2024llama3} as the policy backbone. Before RL, we fine-tune the policy on 2,000 examples following the SFT setup of~\citet{li2025chain}. Following the Search-R1 setup~\citep{jin2025searchr1}, RL training uses Natural Questions (NQ)~\citep{kwiatkowski-etal-2019-natural} and HotpotQA~\citep{yang2018hotpotqadatasetdiverseexplainable}. Retrieval uses the 2018 Wikipedia corpus~\citep{karpukhin-etal-2020-dense} indexed with E5-base-v2 embeddings~\citep{wang2024textembeddingsweaklysupervisedcontrastive}, returning the top three documents per query. We evaluate on seven benchmarks: the single-hop datasets NQ, TriviaQA~\citep{joshi2017triviaqalargescaledistantly}, and PopQA~\citep{mallen2023trustlanguagemodelsinvestigating}, and the multi-hop datasets HotpotQA, 2WikiMultiHopQA~\citep{ho2020constructingmultihopqadataset}, MuSiQue~\citep{trivedi2021musique}, and Bamboogle~\citep{press2023measuringnarrowingcompositionalitygap}.

\subsection{Joint Parameter Sweeps}
Figure~\ref{fig:coupled_sweeps} reveals an interaction between policy- and advantage-side correction across all staleness levels. When $\eta_P=0.25$, increasing $\eta_A$ from $0.25$ to $0.75$ improves performance, whereas the change degrades performance when $\eta_P=0.75$. This reversal persists across all staleness levels, demonstrating that the effect of adjusting one correction channel depends on the setting of the other.

The sweeps also reveal a compensatory structure: stronger correction in one channel tends to favor weaker correction in the other. Importantly, high performance is not confined to a narrowly tuned parameter pair; a broad range of correction strengths remains competitive across all tested staleness levels, indicating substantial hyperparameter robustness. Together, these results support the coordination hypothesis in Section~\ref{sec:theory} and motivate joint calibration of the two channels, while enabling a shared setting across subsequent experiments without staleness-specific retuning.

\subsection{Main Results and Ablations}
Tables~\ref{tab:retool_main_results} and~\ref{tab:searchr1_main_results} report results on tool-integrated mathematical reasoning and search. COPC achieves the highest average performance across settings. On mathematical reasoning, it attains the best result on all five benchmarks, reaching an average of $50.29$. On search, COPC achieves the highest average accuracy of $44.57$ and outperforms the strongest asynchronous baseline by $2.77$ points. This advantage is aligned with COPC's two-channel design: rather than correcting policy mismatch alone, COPC controls policy-side updates and continuation-dependent return and advantage estimation.

The ablation COPC w/o RC replaces corrected returns and advantages with standard GAE while retaining other settings. Return correction improves performance from $48.11$ to $50.29$ on mathematical reasoning and from $40.80$ to $44.57$ on search, outperforming the ablation on all five mathematical-reasoning benchmarks and six of seven search benchmarks. These gains show that policy-side masking alone cannot explain COPC's improvement; advantage-side correction provides a complementary benefit, consistent with the advantage-staleness motivation in Section~\ref{sec:preliminaries}.

\subsection{Training Stability, Robustness, and Efficiency}
Figure~\ref{fig:stability_robustness} evaluates training stability and robustness to increasing policy staleness. On search, COPC remains stable throughout training, whereas most evaluated asynchronous baselines exhibit late-stage performance collapse. COPC w/o RC also degrades sharply late in training, suggesting that advantage-side correction contributes not only to final performance but also to training stability. On AIME 2025, COPC remains stable as staleness increases from $s=8$ to $s=64$, with comparable final avg@32 performance across all settings, demonstrating robustness to highly stale rollouts.

Across both task settings, COPC closely matches the efficiency of asynchronous PPO, incurring only about a $1\%$ increase in step time and at most a $2.5\%$ reduction in throughput, while retaining an approximately $1.7\times$ step-time speedup over synchronous PPO. Together, these results show that COPC improves the robustness of asynchronous training without sacrificing its efficiency advantage.

\begin{table*}[t]
    \centering
    \caption{Training efficiency on tool-integrated mathematical reasoning and search.}
    \label{tab:training_efficiency}
    \vspace{3pt}
    \begin{tabular*}{\textwidth}{@{\extracolsep{\fill}}lccc@{}}
        \toprule
        & Step time (s) & Throughput (tokens/s/GPU) & Step-time speedup \\
        \midrule
        \multicolumn{4}{@{}l}{\textit{Tool-integrated mathematical reasoning}} \\
        \addlinespace
        Synchronous PPO & 410.59 & 614.96 & $1.00\times$ \\
        Asynchronous PPO & 236.80 & 856.93 & $1.73\times$ \\
        \textbf{COPC (ours)} & 238.95 & 836.85 & $1.72\times$ \\
        \midrule
        \multicolumn{4}{@{}l}{\textit{Search}} \\
        \addlinespace
        Synchronous PPO & 213.18 & 65.35 & $1.00\times$ \\
        Asynchronous PPO & 124.90 & 96.44 & $1.71\times$ \\
        \textbf{COPC (ours)} & 126.18 & 94.92 & $1.69\times$ \\
        \bottomrule
    \end{tabular*}
\end{table*}

%% file: sections/conclusion.tex
\section{Conclusion}
We show that correcting stale trajectories in asynchronous LLM reinforcement learning requires more than controlling policy mismatch in the actor update. We identify \emph{advantage staleness}, arising from behavior-policy continuations embedded in return and advantage estimates, and formulate stale-trajectory correction through complementary policy- and advantage-side channels. Our bias and variance decompositions reveal nonseparable interactions between the two channels, motivating their joint calibration. Building on this insight, we introduce COPC, which combines token-level policy-side masking with advantage-side correction based on two-sided clipped-ratio weighting of TD residuals. Across tool-integrated mathematical reasoning and search, COPC improves performance and training stability, remains robust across correction strengths and substantial policy staleness, and incurs minimal additional overhead while retaining a $1.7\times$ step-time speedup over synchronous PPO. More broadly, our results suggest that off-policy correction in asynchronous LLM RL should be treated as a coupled estimation problem rather than solely as policy-ratio control.

%% file: sections/appendix.tex
\section{Proof of Proposition~\ref{prop:return_fixed_point}}
\label{app:return_fixed_point}

Fix the learner policy $\pi_\theta$ and behavior policy $\mu$. Let $(\gamma,\lambda)\in[0,1]^2$ and $(\epsilon_\ell,\epsilon_h)\in[0,1]\times[0,\infty)$ be fixed. Consider a finite-horizon process with time-indexed states and episode length $T=|y|\le H$ almost surely, where $H\in\mathbb N$. Value functions satisfy $V(s_{T+1})=0$ at terminal states. For every nonterminal state $s$, assume $\operatorname{supp}\pi_\theta(\cdot\mid s)\subseteq\operatorname{supp}\mu(\cdot\mid s)$. All expectations are assumed finite.

\begin{proof}
For $a\in\operatorname{supp}\mu(\cdot\mid s)$, define
\begin{equation*}
    \bar\rho(s,a)=\operatorname{clip}\!\left(\frac{\pi_\theta(a\mid s)}{\mu(a\mid s)},1-\epsilon_\ell,1+\epsilon_h\right),
    \qquad
    q(s,a)=
    \begin{cases}
        \mu(a\mid s)\bar\rho(s,a), & a\in\operatorname{supp}\mu(\cdot\mid s),\\
        0, & \text{otherwise}.
    \end{cases}
\end{equation*}
By the support assumption,
\begin{equation*}
    q(s,a)=\operatorname{clip}\!\left(\pi_\theta(a\mid s),(1-\epsilon_\ell)\mu(a\mid s),(1+\epsilon_h)\mu(a\mid s)\right)
\end{equation*}
for every action $a$. Define $\kappa(s)=\sum_a q(s,a)$. Since $\pi_\theta(\cdot\mid s)$ is a probability distribution and its support is contained in that of $\mu(\cdot\mid s)$, there exists at least one action with $q(s,a)>0$. Hence $\kappa(s)>0$ at every nonterminal state.

Write $\bar\rho_j=\bar\rho(s_j,y_j)$ and $\delta_j^V=r_j+\gamma V(s_{j+1})-V(s_j)$. The expected corrected-return operator is
\begin{equation}
\label{eq:proof_return_operator}
    (\mathcal R V)(s)=V(s)+\mathbb E_\mu\!\left[\sum_{j=t}^{T}(\gamma\lambda)^{j-t}\bar\rho_j\delta_j^V\,\middle|\,s_t=s\right].
\end{equation}

We first establish uniqueness. Suppose $V_1$ and $V_2$ are fixed points of $\mathcal R$, and let $U=V_1-V_2$. Subtracting their fixed-point equations gives
\begin{equation}
\label{eq:proof_fixed_point_difference}
    \mathbb E_\mu\!\left[\sum_{j=t}^{T}(\gamma\lambda)^{j-t}\bar\rho_j\bigl(\gamma U(s_{j+1})-U(s_j)\bigr)\,\middle|\,s_t=s\right]=0.
\end{equation}
We proceed by backward induction over the time index. At the latest possible nonterminal time $H$, all successor states are terminal, so~\eqref{eq:proof_fixed_point_difference} reduces to $-\kappa(s)U(s)=0$. Since $\kappa(s)>0$, we have $U(s)=0$ for every state at time $H$. Now suppose $U$ vanishes on all states at times later than $t$. For any state $s$ at time $t$, every future-state term in~\eqref{eq:proof_fixed_point_difference} vanishes, again leaving $-\kappa(s)U(s)=0$. Thus $U(s)=0$. Repeating backward over all time indices yields $U\equiv0$, so the fixed point is unique.

It remains to identify the fixed point. Define
\begin{equation*}
    \widetilde\pi(a\mid s)=\frac{q(s,a)}{\kappa(s)}
    =\frac{\mu(a\mid s)\bar\rho(s,a)}{\sum_b\mu(b\mid s)\bar\rho(s,b)},
\end{equation*}
where the second equality uses the zero extension outside $\operatorname{supp}\mu(\cdot\mid s)$. Since $\kappa(s)>0$, $\widetilde\pi$ is well defined and coincides with the policy in~\eqref{eq:effective_correction_policy}.

Let $r(s,a)$ denote the expected immediate reward and $P(\cdot\mid s,a)$ the transition kernel. At $V^{\widetilde\pi}$,
\begin{align*}
    \mathbb E_\mu\!\left[\bar\rho_t\delta_t^{V^{\widetilde\pi}}\,\middle|\,s_t=s\right]
    &=\sum_a q(s,a)\left[r(s,a)+\gamma\mathbb E_{s'\sim P(\cdot\mid s,a)}V^{\widetilde\pi}(s')-V^{\widetilde\pi}(s)\right] \\
    &=\kappa(s)\left[\sum_a\widetilde\pi(a\mid s)\left(r(s,a)+\gamma\mathbb E_{s'\sim P(\cdot\mid s,a)}V^{\widetilde\pi}(s')\right)-V^{\widetilde\pi}(s)\right] \\
    &=0,
\end{align*}
where the final equality follows from the finite-horizon Bellman equation for $\widetilde\pi$.

Therefore, each weighted TD residual has zero conditional expectation at its corresponding state. Applying the tower property to each term in~\eqref{eq:proof_return_operator} gives
\begin{equation*}
    (\mathcal R V^{\widetilde\pi})(s)-V^{\widetilde\pi}(s)
    =\mathbb E_\mu\!\left[\sum_{j=t}^{T}(\gamma\lambda)^{j-t}\bar\rho_j\delta_j^{V^{\widetilde\pi}}\,\middle|\,s_t=s\right]
    =0.
\end{equation*}
Hence $V^{\widetilde\pi}$ is a fixed point of $\mathcal R$. Uniqueness then completes the proof.
\end{proof}

\begin{corollary}[Convergence of the expected return iteration]
\label{cor:return_iteration_convergence}
Under the assumptions of Proposition~\ref{prop:return_fixed_point}, suppose additionally that the time-indexed nonterminal state space is finite. Fix $\pi_\theta$, $\mu$, and the clipping thresholds. Then, for any initial value function $V^{(0)}$ satisfying the terminal condition, the exact iteration $V^{(n+1)}=\mathcal R V^{(n)}$ converges to $V^{\widetilde\pi}$:
\begin{equation*}
    \lim_{n\to\infty}\|V^{(n)}-V^{\widetilde\pi}\|_\infty=0.
\end{equation*}
\end{corollary}

\begin{proof}
Let $\ell=1-\epsilon_\ell\in[0,1]$ and $u=1+\epsilon_h\in[1,\infty)$. A direct case analysis gives
\begin{equation*}
    \operatorname{clip}(\xi,\ell,u)\le \ell+\left(1-\frac{\ell}{u}\right)\xi
    \qquad\text{for all }\xi\ge0.
\end{equation*}
By the support assumption, $\mathbb E_\mu[\rho_t\mid s_t=s]=1$. Therefore,
\begin{equation}
\label{eq:clipped_ratio_mean_bound}
    0<\kappa(s)=\mathbb E_\mu[\bar\rho_t\mid s_t=s]
    \le 1+\ell\left(1-\frac{1}{u}\right)<2.
\end{equation}

Represent value functions on the finite set of nonterminal states as vectors ordered by increasing time index. Since $\mathcal R$ is affine in $V$, write $\mathcal R V=b+MV$, where $b$ is independent of $V$. From~\eqref{eq:proof_return_operator}, the update at each state depends only on its own value and values at later time indices. Hence $M$ is upper triangular. The coefficient multiplying the current-state value is
\begin{equation*}
    1-\mathbb E_\mu[\bar\rho_t\mid s_t=s]=1-\kappa(s),
\end{equation*}
so the diagonal entries are $M_{ss}=1-\kappa(s)$. Therefore, writing $r(M)$ for the spectral radius of $M$,
\begin{equation*}
    r(M)=\max_s|1-\kappa(s)|<1,
\end{equation*}
where the strict inequality follows from~\eqref{eq:clipped_ratio_mean_bound} and finiteness of the state space.

By Proposition~\ref{prop:return_fixed_point}, $V^{\widetilde\pi}=b+MV^{\widetilde\pi}$. Subtracting this fixed-point identity from the iteration yields
\begin{equation*}
    V^{(n)}-V^{\widetilde\pi}=M^n\bigl(V^{(0)}-V^{\widetilde\pi}\bigr).
\end{equation*}
Since $r(M)<1$, $M^n\to0$ in any matrix norm. Consequently,
\begin{equation*}
    \|V^{(n)}-V^{\widetilde\pi}\|_\infty
    \le \|M^n\|_\infty\|V^{(0)}-V^{\widetilde\pi}\|_\infty
    \to0.
\end{equation*}
\end{proof}

\section{Proofs for Coupling Between Correction Channels}
\label{proof:couple_bias_variance}

This appendix proves the results in Section~\ref{sec:theory}. We fix the learner parameters $\theta$, the value estimator, and the behavior-policy sampling scheme, and consider an arbitrary fixed pair of correction settings $(\eta_P,\eta_A)$. Expectations are taken over the sampled prompt $x\sim\cD$, the response $y\sim\mu(\cdot\mid x)$, and any additional randomness used to construct the advantage estimates, including co-sampled responses when the estimator uses a prompt group.

Responses are nonempty and have finite length almost surely. We assume that all terms in the expectation decompositions are absolutely integrable and that all quantities appearing in the variance results have finite second moments. No independence is assumed between tokens within a response or between responses within the same prompt group unless explicitly stated.

\subsection{Proof of Proposition~\ref{prop:bias_decomposition}}

\begin{proof}
Consider a realization of a prompt, its response, and any additional samples used to construct the advantage estimates. For each valid token position $t\in\{1,\ldots,|y|\}$, recall
\begin{equation*}
    \Delta_t(\eta_P)=w_t(\eta_P)-\rho_t(\theta),\qquad
    e_t(\eta_A)=\hat A_t(\eta_A)-A_t^{\pi_\theta}.
\end{equation*}
Substituting these definitions gives the samplewise identity
\begin{equation*}
\begin{aligned}
    w_t(\eta_P)\hat A_t(\eta_A)-\rho_t(\theta)A_t^{\pi_\theta}
    &=\bigl(\rho_t(\theta)+\Delta_t(\eta_P)\bigr)
      \bigl(A_t^{\pi_\theta}+e_t(\eta_A)\bigr)
      -\rho_t(\theta)A_t^{\pi_\theta} \\
    &=\Delta_t(\eta_P)A_t^{\pi_\theta}
      +\rho_t(\theta)e_t(\eta_A)
      +\Delta_t(\eta_P)e_t(\eta_A).
\end{aligned}
\end{equation*}
The first two terms capture the direct errors from the policy- and advantage-side channels, respectively, while the third captures their interaction.

Multiplying by the score $\psi_t(\theta)$, summing over valid token positions, and applying the response-length normalization yields
\begin{equation*}
\begin{aligned}
    &\frac{1}{|y|}\sum_{t=1}^{|y|}\psi_t(\theta)
    \bigl[w_t(\eta_P)\hat A_t(\eta_A)-\rho_t(\theta)A_t^{\pi_\theta}\bigr] \\
    &\qquad=
    \frac{1}{|y|}\sum_{t=1}^{|y|}\psi_t(\theta)
    \bigl[
        \Delta_t(\eta_P)A_t^{\pi_\theta}
        +\rho_t(\theta)e_t(\eta_A)
        +\Delta_t(\eta_P)e_t(\eta_A)
    \bigr].
\end{aligned}
\end{equation*}
Taking expectations under the common sampling law and using the definitions of $g(\eta_P,\eta_A)$ and $g^\star(\theta)$ in~\eqref{eq:expected_two_sided_gradient} and~\eqref{eq:reference_gradient} gives
\begin{equation*}
    g(\eta_P,\eta_A)-g^\star(\theta)
    =
    \BE\!\left[
        \frac{1}{|y|}\sum_{t=1}^{|y|}
        \psi_t(\theta)
        \bigl(
            \Delta_t(\eta_P)A_t^{\pi_\theta}
            +\rho_t(\theta)e_t(\eta_A)
            +\Delta_t(\eta_P)e_t(\eta_A)
        \bigr)
    \right].
\end{equation*}
The random response length remains inside the samplewise expression, so the argument applies directly to responses of varying lengths. By the stated integrability assumptions, all terms are well defined. Since the decomposition holds before taking expectations, it requires no independence between the two correction errors, token contributions, or co-sampled responses. This proves Proposition~\ref{prop:bias_decomposition}.
\end{proof}

\subsection{Proof of Proposition~\ref{prop:variance_decomposition}}

\begin{proof}
Fix a token position $t$ and correction settings $(\eta_P,\eta_A)$. Recall that $C_t$ denotes conditioning information under which $|y|$, $\psi_t(\theta)$, and $w_t(\eta_P)$ are measurable for every policy-side setting $\eta_P$. Define
\begin{equation*}
    Z_t=\frac{w_t(\eta_P)\psi_t(\theta)\hat A_t(\eta_A)}{|y|}.
\end{equation*}
For positions beyond the realized response length, we set $w_t(\eta_P)$, $\psi_t(\theta)$, and $\hat A_t(\eta_A)$ to zero and assign fixed dummy values to $(s_t,y_t)$. Thus, all token-level quantities are well defined for a fixed index $t$ across responses of varying lengths.

By construction, $w_t(\eta_P)\psi_t(\theta)/|y|$ is measurable with respect to $C_t$, and the same conditioning information is used across correction settings. Under the standing moment assumptions, $\hat A_t(\eta_A)$ and $Z_t$ are square-integrable. Recall
\begin{equation*}
    m_t(\eta_A)=\BE[\hat A_t(\eta_A)\mid C_t],\qquad
    \sigma_t^2(\eta_A)=\operatorname{Var}[\hat A_t(\eta_A)\mid C_t].
\end{equation*}
Because padded advantage values are identically zero, both conditional moments vanish at padded positions.

Using the $C_t$-measurability of the coefficient,
\begin{equation*}
    \BE[Z_t\mid C_t]
    =\frac{w_t(\eta_P)\psi_t(\theta)m_t(\eta_A)}{|y|}.
\end{equation*}
Hence
\begin{equation*}
    Z_t-\BE[Z_t\mid C_t]
    =\frac{w_t(\eta_P)\psi_t(\theta)}{|y|}
    \bigl(\hat A_t(\eta_A)-m_t(\eta_A)\bigr),
\end{equation*}
and therefore
\begin{equation*}
    \BE\!\left[\bigl\|Z_t-\BE[Z_t\mid C_t]\bigr\|_2^2\mid C_t\right]
    =\frac{w_t^2(\eta_P)\|\psi_t(\theta)\|_2^2\sigma_t^2(\eta_A)}{|y|^2}.
\end{equation*}

To obtain the unconditional variance, write
\begin{equation*}
    Z_t-\BE[Z_t]
    =\bigl(Z_t-\BE[Z_t\mid C_t]\bigr)
    +\bigl(\BE[Z_t\mid C_t]-\BE[Z_t]\bigr).
\end{equation*}
The first term has zero conditional mean given $C_t$, whereas the second is $C_t$-measurable, so their expected inner product is zero. Expanding the squared norm therefore gives the vector law of total variance:
\begin{equation}
\label{eq:total_variance_appendix}
    \mathbb V[Z_t]
    =\mathbb V[\BE[Z_t\mid C_t]]
    +\BE\!\left[\bigl\|Z_t-\BE[Z_t\mid C_t]\bigr\|_2^2\right].
\end{equation}
Substituting the conditional mean and applying the tower property yields
\begin{equation*}
    \mathbb V[Z_t]
    =\mathbb V\!\left[\frac{w_t(\eta_P)\psi_t(\theta)m_t(\eta_A)}{|y|}\right]
    +\BE\!\left[\frac{w_t^2(\eta_P)\|\psi_t(\theta)\|_2^2\sigma_t^2(\eta_A)}{|y|^2}\right].
\end{equation*}
This is the claimed decomposition. The argument relies only on conditional measurability and the stated moment assumptions; no independence between the policy-side weight and the advantage estimate is required.
\end{proof}

\subsection{Response-Level Variance Decomposition}
\label{app:trajectory_variance}

Assume additionally that $|y|\le T_{\max}$ almost surely for some finite $T_{\max}$. Under the zero-padding convention, the response-level gradient contribution satisfies
\begin{equation*}
    Z=\sum_{t=1}^{|y|}Z_t=\sum_{t=1}^{T_{\max}}Z_t.
\end{equation*}
Its total variance decomposes as
\begin{equation}
\label{eq:trajectory_variance_decomposition}
    \mathbb V[Z]
    =\sum_{t=1}^{T_{\max}}\mathbb V[Z_t]
    +2\sum_{1\le t<j\le T_{\max}}
    \BE\!\left[(Z_t-\BE[Z_t])^\top(Z_j-\BE[Z_j])\right].
\end{equation}

\begin{proof}
Since the sum is finite, linearity of expectation gives
\begin{equation*}
    \BE[Z]=\sum_{t=1}^{T_{\max}}\BE[Z_t],\qquad
    Z-\BE[Z]=\sum_{t=1}^{T_{\max}}\bigl(Z_t-\BE[Z_t]\bigr).
\end{equation*}
Expanding the squared Euclidean norm yields
\begin{equation*}
    \bigl\|Z-\BE[Z]\bigr\|_2^2
    =\sum_{t=1}^{T_{\max}}\bigl\|Z_t-\BE[Z_t]\bigr\|_2^2
    +2\sum_{1\le t<j\le T_{\max}}
    \bigl(Z_t-\BE[Z_t]\bigr)^\top
    \bigl(Z_j-\BE[Z_j]\bigr).
\end{equation*}
The finite second-moment assumptions ensure that all terms are integrable, with cross terms controlled by Cauchy--Schwarz. Taking expectations and using $\mathbb V[X]=\BE[\|X-\BE[X]\|_2^2]$ gives~\eqref{eq:trajectory_variance_decomposition}.
\end{proof}

Each $Z_t$ retains the random response-length normalization and vanishes beyond termination. Thus, the decomposition captures both cross-token dependence and variability induced by response length, without assuming independence among token contributions.

\subsection{Proof of Corollary~\ref{cor:joint_gradient_error}}

\begin{proof}
For prompt group $i$, let
\begin{equation*}
    \bar Z_{i,K}=\frac{1}{K}\sum_{k=1}^{K}Z_i^{(k)}
\end{equation*}
denote its average response-level gradient contribution. By definition,
\begin{equation*}
    \hat g_{N,K}=\frac{1}{N}\sum_{i=1}^{N}\bar Z_{i,K}.
\end{equation*}
Each response contribution has marginal expectation $g(\eta_P,\eta_A)$, and hence
\begin{equation*}
    \BE[\bar Z_{i,K}]=g(\eta_P,\eta_A),
    \qquad
    \BE[\hat g_{N,K}]=g(\eta_P,\eta_A).
\end{equation*}
For brevity, write $g=g(\eta_P,\eta_A)$ and $g^\star=g^\star(\theta)$. Then
\begin{equation*}
    \hat g_{N,K}-g^\star=(\hat g_{N,K}-g)+(g-g^\star).
\end{equation*}
Since $\BE[\hat g_{N,K}-g]=0$, the cross term vanishes, giving
\begin{equation}
\label{eq:mse_decomposition_appendix}
    \BE\!\left[\|\hat g_{N,K}-g^\star\|_2^2\right]
    =\BE\!\left[\|\hat g_{N,K}-g\|_2^2\right]
    +\|g-g^\star\|_2^2.
\end{equation}

The $N$ prompt groups are i.i.d., while arbitrary dependence among the $K$ responses within each group is retained in $\bar Z_{i,K}$. Therefore,
\begin{equation*}
\begin{aligned}
    \BE\!\left[\|\hat g_{N,K}-g\|_2^2\right]
    =\BE\!\left[
        \left\|
        \frac{1}{N}\sum_{i=1}^{N}(\bar Z_{i,K}-g)
        \right\|_2^2
    \right] 
    =\frac{1}{N^2}\sum_{i=1}^{N}
    \BE\!\left[\|\bar Z_{i,K}-g\|_2^2\right]
    =\frac{1}{N}\mathbb V[\bar Z_K].
\end{aligned}
\end{equation*}
Thus, independence is required only across prompt groups; no independence is assumed among responses within the same group.

By Proposition~\ref{prop:bias_decomposition}, the bias term
$g(\eta_P,\eta_A)-g^\star(\theta)$ admits the exact decomposition in~\eqref{eq:exact_bias_decomposition}. Substituting the variance identity into~\eqref{eq:mse_decomposition_appendix} therefore yields
\begin{equation*}
    \cE(\eta_P,\eta_A)
    =
    \bigl\|g(\eta_P,\eta_A)-g^\star(\theta)\bigr\|_2^2
    +\frac{1}{N}\mathbb V[\bar Z_K(\eta_P,\eta_A)],
\end{equation*}
which proves the corollary.

\end{proof}

\newpage
\section{Detailed Training Configurations}
\label{ape:detailed_training_configurations}

\subsection{Prompt Templates and Reward Design}
\label{app:prompts_rewards}

\paragraph{Tool-integrated mathematical reasoning.}
For tool-augmented mathematical reasoning, we use the Qwen chat template with the built-in Python code-interpreter instructions, shown below, for training and evaluation. During rollout, the assistant may issue structured tool calls; each call is executed and its output is appended to the conversation context before generation resumes, allowing subsequent reasoning to incorporate intermediate computational results.

We use a binary correctness reward. Let $z$ denote the text passed to the mathematical verifier and $a^\star$ the reference answer. We use \texttt{math\_verify} to extract candidate mathematical expressions from $z$ and determine their equivalence to $a^\star$:
\begin{equation}
\label{eq:math_reward}
r_{\mathrm{math}}(z,a^\star)=
\begin{cases}
1, & \text{if an extracted answer is verified as equivalent to } a^\star,\\
0, & \text{otherwise}.
\end{cases}
\end{equation}
No separate formatting reward or tool-use bonus is applied; tool calls affect the reward only through their contribution to a verifiably correct final answer.
\begin{promptbox}{System Prompt}
You are Qwen, created by Alibaba Cloud. You are a helpful assistant.\par
\texttt{\# Tools}\par
You may call one or more functions to assist with the user query.\par
You are provided with function signatures within \texttt{<tools></tools>} XML tags:
\begin{lstlisting}[style=promptcode]
<tools>
{"type": "function",
 "function": {
   "name": "code_interpreter",
   "description": "A tool for executing code.",
   "parameters": {
     "type": "object",
     "properties": {
       "code": {"type": "string", "description": "The code to execute."}
     },
     "required": ["code"]
   }
 }}
</tools>
\end{lstlisting}
For each function call, return a JSON object with function name and arguments within \texttt{<tool\_call></tool\_call>} XML tags:
\begin{lstlisting}[style=promptcode]
<tool_call>
{"name": <function-name>, "arguments": <args-json-object>}
</tool_call>
\end{lstlisting}
\end{promptbox}

\paragraph{Search.}
For search-augmented reasoning, we use the system prompt shown below during both training and evaluation. The prompt specifies the search-tool schema and the required format for structured tool calls.

We use a shaped reward for search-augmented reasoning. Let
$c(s,\mathcal{Y})\in\{0,1\}$ denote normalized exact-match correctness,
$f(s)\in\{0,1\}$ indicate whether a valid boxed answer is present, and
$n$ denote the number of executed search calls. The reward is
\begin{equation}
\label{eq:search_reward}
r_{\mathrm{search}}(s,\mathcal{Y},n)
=
\begin{cases}
2.0, & c(s,\mathcal{Y})=1,\ n>0,\\
1.1, & c(s,\mathcal{Y})=1,\ n=0,\\
-1+0.1f(s)+\min(0.1n,0.3), & c(s,\mathcal{Y})=0.
\end{cases}
\end{equation}
A correct answer receives reward $1.1$ without search and $2.0$ with search.
For incorrect answers, valid formatting contributes $0.1$ and the search-use
bonus is capped at $0.3$, so the total reward remains negative.

\begin{promptbox}{System Prompt}
Given the following functions, please respond with a JSON for a function call
with its proper arguments that best answers the given prompt.\par

Respond in the format "name": function name, "parameters": dictionary of
argument name and its value. Do not use variables.

\begin{lstlisting}[style=promptcode]
{
  "function": {
    "description": "Searches the web for relevant information based on the given query.",
    "name": "search",
    "parameters": {
      "properties": {
        "query_list": {
          "description": "A list of fully-formed semantic queries. The tool will return search results for each query.",
          "items": {"type": "string"},
          "type": "array"
        }
      },
      "required": ["query_list"],
      "type": "object"
    }
  },
  "type": "function"
}
\end{lstlisting}
\end{promptbox}

\subsection{RL Training Configuration}
All experiments were conducted on 8 NVIDIA H200 GPUs. We use Qwen3-4B-Instruct for tool-integrated mathematical reasoning and Llama-3.2-3B-Instruct for search-augmented reasoning.

\paragraph{Supervised Fine-Tuning.}
Before reinforcement learning, we perform supervised fine-tuning (SFT) for each policy backbone. For tool-integrated mathematical reasoning, Qwen3-4B is fine-tuned on the ReTool SFT corpus~\citep{feng2025retool}. For search-augmented reasoning, Llama-3.2-3B-Instruct is fine-tuned on 2,000 examples sampled from the Chain-of-Agents SFT corpus~\citep{li2025chain}. The SFT configurations are summarized in Table~\ref{tab:sft_configuration}.

Table~\ref{tab:task_config} summarizes the default training configurations for the two task settings. Method-specific hyperparameters and deviations, including correction thresholds and update schedules, are specified below. Critic-related settings apply only to methods with a learned value function. Prompt templates and reward functions are described in Section~\ref{app:prompts_rewards}.

\subsubsection{Method-specific configurations.}
\label{ape:method-specific-configurations}

Table~\ref{tab:task_config} summarizes the shared default configurations. Unless otherwise specified, losses are averaged per response for tool-integrated mathematical reasoning and per token for search. Method-specific settings and deviations are described below.

\paragraph{PPO.}
PPO uses the clipped surrogate objective with standard GAE and follows the default configuration in Table~\ref{tab:task_config}.

\paragraph{DAPO.}
DAPO uses group-relative advantages without a learned critic. Rewards are centered and standardized within each prompt group, without additional batch-level advantage normalization. Dynamic sampling retains groups with nonzero reward variance and replenishes discarded groups until the training batch is complete. For tool-integrated mathematical reasoning, DAPO additionally replaces response-level loss averaging with token-level averaging. Overlong reward shaping is disabled.

\paragraph{SAO.}
SAO replaces the PPO clipped surrogate with two-sided importance-ratio masking using the default masking interval in Table~\ref{tab:task_config}. We sample one response per prompt in both task settings, so each training batch contains responses from 512 distinct prompts.

\paragraph{AReaL.}
AReaL uses standard GAE and a decoupled actor objective, clipping the learner-to-proximal ratio to the default interval in Table~\ref{tab:task_config} while applying a separate importance weight to correct proximal-to-behavior mismatch. Tokens whose proximal-to-behavior ratio exceeds $5.0$ are excluded from the actor loss. Each rollout batch contains 1,024 responses and is split into two disjoint minibatches of 512 responses, yielding two actor updates with the proximal policy fixed across both. For search, the actor learning-rate warmup lasts 25 optimizer updates.

\paragraph{V-trace.}
Following~\citet{espeholt2018impala}, V-trace uses truncation thresholds of $1$ for value correction, trace coefficients, and policy-gradient correction, i.e., $\bar{\rho}=\bar{c}=\bar{\rho}_{\mathrm{PG}}=1$. The actor uses the V-trace policy-gradient objective without advantage normalization, while the critic is trained against the corresponding V-trace target without PPO-style value clipping.

\paragraph{COPC.}
COPC uses base clipping margins $\epsilon_\ell=0.4$ and $\epsilon_h=0.56$, with interval controls $(\eta_P,\eta_A)=(0.5,0.5)$ in the main experiments. This gives the interval $[0.80,1.28]$ for both policy-side masking and advantage-side ratio clipping. COPC w/o RC sets $\eta_A=0$, recovering standard GAE on the advantage side, while retaining $\eta_P=0.5$ and all other settings.

\subsubsection{Inference configurations.}

\paragraph{Tool-integrated mathematical reasoning.}
We sample 32 responses per problem using a temperature of $0.6$ and top-$p=1.0$. Each response has a total token budget of 38,912 across all interaction turns, including tool observations. We report avg@32 accuracy on AIME24, AIME25, AIME26, BRUMO25, and BeyondAIME.

\paragraph{Search.}
We use greedy decoding with a temperature of $0$ and top-$p=1.0$, generating one response per question with a maximum response budget of 8,192 tokens. We report normalized exact-match accuracy on NQ, TriviaQA, PopQA, HotpotQA, 2WikiMultiHopQA, MuSiQue, and Bamboogle.

\begin{table}[t]
\centering
\caption{Supervised fine-tuning (SFT) configurations.}
\vspace{4pt}
\label{tab:sft_configuration}
\renewcommand{\arraystretch}{1.2}
    \begin{tabular}{ll}
        \toprule
        Hyperparameter & Value \\
        \midrule
        Optimizer & AdamW \\
        Learning rate & $1\times10^{-5}$ \\
        Weight decay & 0.01 \\
        Learning-rate schedule & cosine \\
        Warmup ratio & 0.1 \\
        Global batch size & 16 \\
        Maximum sequence length & 16,384 (math) / 8,192 (search) \\
        Training epochs & 2 \\
        \bottomrule
    \end{tabular}
\end{table}

\begin{table}[t]
\centering
\caption{Default training configurations for tool-integrated mathematical reasoning and search. Rollout generation and policy optimization are executed asynchronously on separate worker groups. Method-specific deviations are described in the accompanying text.}
\label{tab:task_config}
\renewcommand{\arraystretch}{1.08}
    \begin{tabular*}{\linewidth}{@{\extracolsep{\fill}}ccc@{}}
    \toprule
    Configuration & \textbf{Math-with-Tool} & \textbf{Search} \\
    \midrule
    \multicolumn{3}{@{}l}{\textit{General training hyperparameters}} \\
    \addlinespace
        Policy model & Qwen3-4B-Instruct & Llama-3.2-3B-Instruct \\
        Samples per prompt & 8 & 4 \\
        Rollout batch size (responses) & 512 & 512 \\
        Training minibatch size (responses) & 512 & 512 \\
        Actor updates per rollout batch & 1 & 1 \\
        Max prompt length (tokens) & 2048 & 512 \\
        Max response length (tokens) & 16384 & 8192 \\
        Max assistant turns & 8 & 8 \\
        Actor learning rate & $1\times10^{-6}$ & $1\times10^{-6}$ \\
        Max staleness $S$ & 8 & 14 \\
        Actor ratio clip interval & $[0.80,1.28]$ & $[0.80,1.28]$ \\
    \midrule
    \multicolumn{3}{@{}l}{\textit{Actor--critic settings}} \\
    \addlinespace
        Critic pretraining samples & $512\times16$ & $512\times5$ \\
        Critic pretraining steps & 50 & 50 \\
        Critic learning rate & $5\times10^{-6}$ & $5\times10^{-6}$ \\
        Discount factor $\gamma$ & 1 & 1 \\
        Trace parameter $\lambda$ & 1 & 1 \\
        Value clipping margin & 0.2 & 0.2 \\
    \bottomrule
    \end{tabular*}
\end{table}